\documentclass[11pt]{article}

\usepackage{latexsym,mathrsfs}
\usepackage{amsmath,amssymb}
\usepackage{amsthm,enumerate,verbatim}
\usepackage{amsfonts}
\usepackage{graphicx}
\usepackage{algorithm}
\usepackage{algorithmic}
\usepackage{mcode}
\usepackage{url}

\newtheorem{lemma}{Lemma}

\newtheorem{theorem}{Theorem}
\newtheorem{corollary}{Corollary}

\newtheorem{assumption}{Assumption}
\newtheorem{remark}{Remark}

\DeclareMathOperator{\conv}{conv} 
\DeclareMathOperator{\diag}{diag} 
\DeclareMathOperator{\argmax}{argmax} 
\DeclareMathOperator{\argmin}{argmin}

\DeclareMathOperator{\col}{col}

\title{Semidefinite Programming Based Preconditioning for More Robust Near-Separable Nonnegative Matrix Factorization\thanks{This work was supported in part by a grant from the U.S. Air Force Office of Scientific Research and a Discovery Grant from the Natural Science and Engineering Research Council (Canada).}}
\date{}

\author{Nicolas Gillis \\ 
Facult\'e Polytechnique, Universit\'e de Mons \\ 
Department of Mathematics and Operational Research \\ 
Rue de Houdain 9, 7000 Mons, Belgium\\
Email: nicolas.gillis@umons.ac.be 
 \and 
 Stephen A. Vavasis  \\ 
University of Waterloo \\ 
Department of Combinatorics and Optimization \\
Waterloo, Ontario N2L 3G1, Canada\\
Email: vavasis@math.uwaterloo.ca}

\begin{document}

\maketitle

\begin{abstract} 
Nonnegative matrix factorization (NMF) under the separability assumption can provably be solved efficiently, even in the presence of noise, and has been shown to be a powerful technique in document classification and hyperspectral unmixing. This problem is referred to as near-separable NMF and requires that there exists a cone spanned by a small subset of the columns of the input nonnegative matrix approximately containing all columns. In this paper, we propose a preconditioning based on semidefinite programming making the input matrix well-conditioned. This in turn can improve significantly the performance of near-separable NMF algorithms which is illustrated on the popular successive projection algorithm (SPA). The new preconditioned SPA is provably more robust to noise, and outperforms SPA on several synthetic data sets. We also show how an active-set method allow us to apply the preconditioning on large-scale real-world hyperspectral images.  

\end{abstract} 


\section{Introduction}

Nonnegative matrix factorization (NMF) has become a standard technique in machine learning and data analysis. NMF is a linear dimensionality reduction technique for nonnegative data where both the basis elements and the weights of the linear combinations are imposed to be nonnegative: 
Given an $m$-by-$n$ nonnegative matrix $M$ and a factorization rank $r$, 
NMF produces an $m$-by-$r$ nonnegative matrix $W$ and an $r$-by-$n$ nonnegative matrix $H$ such that $WH \approx M$. 
The columns of $M$, which usually represent elements of a data set such as images or documents, are approximately reconstructed using nonnegative linear combinations of the columns of $W$ since $M(:,j) \approx \sum_{k=1}^r W(:,k) H(k,j)$ for all $j$. 
The advantage of the nonnegativity constraints is twofold. 
First, the nonnegativity of the basis elements (that is, the columns of $W$) allows us to interpret them in the same way as the data (e.g., as images or documents). 
Second, the nonnegativity of the weights in the linear combinations only allows an additive reconstruction of the data points from the basis elements leading to a parts-based and sparse representation of the data~\cite{LS99}. 

Unfortunately, NMF is NP-hard \cite{V09} and highly ill-posed \cite{G12}. Therefore, in practice, people usually use standard nonlinear optimization techniques to find locally optimal solutions. Hence most NMF algorithms come with no guarantee. However, NMF algorithms have been proved successful in a series of applications which suggests that some real-world NMF problems might not be as difficult as the general NMF problem. In fact, it was recently shown by Arora et al.~\cite{AGKM11} that there exists a subclass of nonnegative matrices, referred to as separable, for which the NMF problem can be solved in polynomial time. Separability requires that there exists an NMF $(W,H)$ of the input matrix $M$ such that $M = WH$ and where each column of $W$ is equal to a column of $M$. 
In other terms, a matrix $M$ is $r$-separable if there exists an index set $\mathcal{K} \subseteq \{1,2,\dots,n\}$ with cardinality $r$ and an $r$-by-$n$ nonnegative matrix $H$ such that $M = M(:,\mathcal{K}) H$. 
This is equivalent to requiring the cone spanned by a subset of $r$ columns of $M$ to contain all columns of $M$. Although this condition is rather strong, it makes sense in several applications. For example, 
\begin{itemize}

\item In document classification, separability requires that for each topic there exists a word used only by that topic (it is referred to as an `anchor' word) \cite{AGM12, Ar13}. 

\item In hyperspectral unmixing, separability requires that for each constitutive material present in the hyperspectral image there exists a pixel containing only that material (it is referred to as a `pure' pixel). The separability assumption coincides with the so called pure-pixel assumption; see Section~\ref{hubsec}.  

\item In blind source separation, separability requires that for each source there exists a moment in time where only that source is active; see, e.g., \cite{CMCW08, CMCW11} and the references therein. 

\end{itemize}

\subsection{Near-Separable NMF} \label{nsnsec}

In practice, the input separable matrix is perturbed with some noise and it is important to design algorithms robust to noise. Note that, in the noiseless case, the separable NMF problem is relatively easy and reduces to identify the extreme rays of the cone spanned by a finite set of points (or the vertices of the convex hull of a set of points after normalization of the columns of input matrix); see, e.g., \cite{KSK12}.  
The separable NMF problem with noise is referred to as near-separable NMF, and can be defined as follows \cite{GL13}. 
\begin{quote}
(Near-Separable NMF) \emph{Given a noisy $r$-separable matrix $\tilde{M} = M + N$ with $M = WH = W[I_r, H'] \Pi$ where $W$ and $H'$ are nonnegative matrices, $I_r$ is the $r$-by-$r$ identity matrix, $\Pi$ is a permutation matrix and $N$ is the noise with $\max_j ||N(:,j)||_2 \leq \epsilon$ for some $\epsilon \geq 0$, find a set $\mathcal{K}$ of $r$ indices such that $\tilde{M}(:,\mathcal{K}) \approx W$. }
\end{quote} 
Several algorithms have been proposed to solve this problem \cite{AGKM11, AGM12, BRRT12, G13, GV12}, and all these algorithms are sensitive to the conditioning of the matrix $W$. 

In this paper, we will assume without loss of generality that the columns of $H$ sum to at most one (this can be obtained by normalizing the columns of the input matrix; see, e.g., \cite{AGKM11}). 
More precisely we assume that the input data matrix has the following form. 
\begin{assumption}[Near-Separable Matrix \cite{GV12}] \label{ass1} 
The separable noiseless matrix $M$ can be written as $M = W  \, H \in \mathbb{R}^{m \times n}$ where $W \in \mathbb{R}^{m \times r}$ has full column rank, $H = [I_r \; H'] \in \mathbb{R}^{r \times n}_+$ and the sum of the entries of each column of $H'$ is at most one.  
The near-separable matrix $\tilde{M} = M + N$ is the perturbed matrix $M$, with 
\[
||N(:,j)||_2 = ||\tilde{M}(:,j)-M(:,j)||_2 \leq \epsilon \quad \text{ for all } j. 
\] 
\end{assumption} 
Finally, given a matrix satisfying {Assumption}~\ref{ass1} whose columns have been arbitrarily permuted, our aim is to automatically identify the columns of $M$ corresponding to the columns of $W$.

\subsection{Successive Projection Algorithm (SPA)} \label{spasec} 

The successive projection algorithm (SPA; see Algorithm~\ref{spa}) is a simple but fast and robust recursive algorithm for solving near-separable NMF: at each step, the column with maximum $\ell_2$ norm is selected and then all columns are projected onto its orthogonal complement. 
It was first introduced in \cite{MC01}, and later proved to be robust in~\cite{GV12}. 
\begin{theorem}[\cite{GV12}, Th.~3] \label{th1} 
Let $\tilde{M}$ satisfy Assumption~\ref{ass1}. If 
$\epsilon \leq \mathcal{O} \left( \,  \frac{  \sigma_{\min}(W)  }{\sqrt{r} \kappa^2(W)} \right)$, then SPA identifies all the columns of $W$ up to error $\mathcal{O} \left( \epsilon \, \kappa^2(W) \right)$, that is, the index set $\mathcal{K}$ identified by SPA satisfies 
\[
\max_{1 \leq j \leq r} \min_{k \in \mathcal{K}} ||W(:,j) - \tilde{M}(:,k)||_2 \leq \mathcal{O} \left( \epsilon \, \kappa^2(W) \right), 
\]
where $\kappa(W) = \frac{\sigma_{\max}(W)}{\sigma_{\min}(W)}$ is the condition number of $W$. 
\end{theorem} 

For SPA, the condition $\epsilon < {\sigma_{\min}(W)}$ is necessary, otherwise $W$ could become rank deficient (take for example  $W = \sigma I_{r}$ and the noise corresponding to the columns of $W$ as $-\sigma I_{r}$). 
The dependence in terms of the condition number $\kappa(W)$ of $W$ in Theorem~\ref{th1} is quite strong: 
for ill-conditioned matrices, $\epsilon$ has to be very close to zero to guarantee recovery. 
\algsetup{indent=2em}
\begin{algorithm}[ht!]
\caption{Successive Projection Algorithm (SPA) \cite{MC01, GV12} \label{spa}}
\begin{algorithmic}[1] 
\REQUIRE Near-separable matrix $M = WH + N \in \mathbb{R}^{m \times n}_+$ (see Assumption~\ref{ass1}),  the number $r$ of columns to be extracted. 
\ENSURE Set of indices $\mathcal{K}$ such that $M(:,\mathcal{K}) \approx W$ (up to permutation). 
    \medskip 
\STATE Let $R = M$, 
$\mathcal{K} = \{\}$, $k=1$.  
\WHILE {$R \neq 0$ and $k \leq r$}   
\STATE $k^* = \argmax_k ||R_{:k}||_2$.  
\STATE $u_k = {R_{:k^*}}$.  \vspace{0.1cm} 
\STATE $R \leftarrow \left(I-\frac{u_k u_k^T}{||u_k||_2^2}\right)R$. \vspace{0.1cm} 
\STATE $\mathcal{K} = \mathcal{K} \cup \{k^*\}$. 
\STATE $k = k+1$.
\ENDWHILE
\end{algorithmic}
\end{algorithm} 

\begin{remark} \label{rem1} 
In Theorem~\ref{th1}, $\kappa(W)$ can actually be replaced with 
$\beta(W) = \frac{\max_i ||W(:,i)||_2}{\sigma_{\min}(W)}$; see \cite{GV12}. 
We choose to work here with $\kappa(W)$ instead of $\beta(W)$  because it is more convenient and makes the presentation nicer. 
Note that $\beta(W) \leq  \kappa(W)  \leq  \sqrt{r}   \beta(W)$ 
so that $\beta(W)$ and $\kappa(W)$  only differ by a factor of at most $\sqrt{r}$ which is usually negligible (in practice, $r$ is in general smaller than 100). 
\end{remark} 

\begin{remark}[Noise Model] \label{rem2} 
Note that the bounds in Theorem~\ref{th1} hold for any bounded noise (but otherwise it can be arbitrary). Moreover, SPA can generalized using other  convex functions for the selection step (step 3 of Algorithm~\ref{spa}), which might improve robustness of SPA depending on the noise model. For example, using $\ell_p$ norms for $1 < p < 2$ (instead of $p = 2$) makes SPA more robust to sparse noise~\cite{GV12}; see also~\cite{AC11} for some numerical experiments.  
\end{remark}

In~\cite{Ar13}, Arora et al.\@ proposed FastAnchorWords, an algorithm closely related to SPA: at each step, instead of picking the column whose projection onto the orthogonal complement of the columns extracted so far has maximum norm, they pick the column whose projection onto the affine hull of the columns extracted so far has maximum norm. 
This variant requires the entries of each column of $H$ to sum to one so that it is less general than SPA\footnote{For example, in hyperspectral imaging, the input matrix usually satisfies the assumption that the entries of each column of $H$ sum to at most one (Assumption~\ref{ass1}), but not to one (for example because of different illumination conditions in the image). Therefore, in this case, FastAnchorWords requires normalization of the input matrix (while SPA does not) and will be rather sensitive to columns of $M$ with small norms (e.g., background pixels); see the discussion in \cite{GV12}.}. 
However, their analysis extends to SPA\footnote{\label{footnote1} In fact, adding the origin in the data set and applying FastAnchorWords by imposing that the origin (which is a vertex) is extracted first makes FastAnchorWords and SPA equivalent.}, 
and they proposed a post-processing to make it more robust: 
Let $\mathcal{K}$ be the set of indices extracted by SPA, and denote $\mathcal{K}(k)$ the index extracted at step $k$. For $k = 1, 2, \dots r$, the post-processing 
\begin{itemize}
\item Projects each column of the data matrix onto the orthogonal complement of $M(:,\mathcal{K} \backslash \{\mathcal{K}(k)\})$. 
\item Identifies the column with maximum norm of the corresponding projected matrix (say the $k'$th), 
\item Updates $\mathcal{K} \, \leftarrow \, \mathcal{K} \backslash \{\mathcal{K}(k)\} \cup \{k'\}$. 
\end{itemize}
See Algorithm~4 in \cite{Ar13} for more details.  
\begin{theorem}[\cite{Ar13}, Th.~4.3] \label{th2} 
Let $\tilde{M}$ satisfy Assumption~\ref{ass1}. If 
$\epsilon \leq \mathcal{O} \left(  \frac{ \sigma_{\min}(W) }{r  \kappa^2(W)} \right)$, then post-processed SPA identifies all the columns of $W$ up to error $\mathcal{O} \left( \epsilon \, \kappa(W) \right)$. 
\end{theorem}

Theorem~\ref{th2} improves upon the error bound\footnote{Note however that the bound on the noise level is slightly weaker than in Theorem~\ref{th1} being proportional to $r^{-1}$ instead of ${r}^{-1/2}$ (although we believe the bound could be improved, and $r$ is small in practice).} of Theorem~\ref{th1} (to a factor $\kappa(W)$), however it still requires $\epsilon$ to be rather small to guarantee any error bound (especially for ill-conditioned matrices). Moreover, it was observed in practice that SPA and its post-processed version often give similar results \cite{Ge} (see also Section~\ref{ne} for some results on synthetic data sets).

\begin{remark}
In \cite{Ar13}, the bounds are in terms of $\gamma(W)$, which is defined as the minimum distance between a vertex and the affine hull of the other vertices. It reads as follows (in our notations):  \cite[Th.~4.3]{Ar13} 
Let $M$ satisfy $||M(:,j)||_1 = 1$ for all $j$. If $\epsilon < \frac{\gamma^3}{20r}$, then there 
is a combinatorial algorithm that given $M$ outputs a subset of the columns of $M$ of size $r$ that are at distance at most $\mathcal{O}( \frac{\epsilon}{\gamma} ) $ from the columns of $W$.

Because we need to add the origin in the data set to extend their analysis to SPA (see footnote~\ref{footnote1}), 
we have to replace $\gamma(W)$ with $\gamma([W, 0]) = \sigma_{\min}(W)$. 
Since $||W(:,j)||_1 = 1$ for all $j$, $1 \leq \sigma_{\max}(W) \leq \sqrt{r}$ and hence 
$\mathcal{O}\left(  \frac{\kappa(W)}{\sqrt{r}} \right) \leq \gamma^{-1}([W, 0]) = \sigma_{\min}(W)^{-1} \leq \mathcal{O}\left( \kappa(W) \right)$ so that 
we have slightly weakened the bound in Theorem~\ref{th2} (of at most a factor of $\sqrt{r}$) to work with $\kappa(W)$. 
\end{remark}

\subsection{Contribution and Outline of the Paper}

The paper is organized as follows. In Section~\ref{precsec}, we propose and analyze a way to precondition near-separable matrices using semidefinite programming (SDP), where we focus on the case $m=r$. 
This in turn allows us to solve near-separable NMF problems in a more robust way since the error bounds provided by near-separable NMF algorithms depend on the conditioning of the input matrix. In particular, this allows us to prove that the preconditioning makes SPA significantly more robust to noise: 
\begin{quote}[Th.~\ref{mainth}]
Let $\tilde{M}$ satisfy Assumption~\ref{ass1} with $m = r$. If $\epsilon \leq \mathcal{O}\left( \frac{\sigma_{\min}(W)}{r\sqrt{r}} \right)$, then preconditioned SPA identifies all the columns of $W$ up to error $\mathcal{O} \left( \epsilon \, \kappa(W) \right)$. 
\end{quote} 
Observe that, up to some factor in $r$, the upper bound on $\epsilon$ only depends on $\sigma_{\min}(W)$, and we cannot expect better for SPA (Section~\ref{spasec}), hence preconditioned SPA can tolerate much higher noise levels than SPA and post-processed SPA.  

In Section~\ref{ldr}, the preconditioning is combined with a linear dimensionality reduction technique to handle the case $m > r$. 
In Section~\ref{as}, an active-set method is proposed to deal with large-scale problems (that is, for $n$ large). 
In Section~\ref{ne}, we show on some synthetic data sets that combining the preconditioning with SPA leads to a significantly more robust algorithm. 
We also show that it can be applied to large-scale real-world hyperspectral images.

\begin{remark} 
While writing this paper, we noticed the very recent paper \cite{M13} (published online two weeks before this paper was, now published as \cite{M13b}) where the same idea is used to solve near-separable NMF problems. 
However, the proposed algorithm and its analysis in \cite{M13} are different: in fact, several SDP's might have to be solved while \emph{the SDP is not used for preconditioning but as a preprocessing to select a subset of the columns of $M$} (see \cite[Algo.~2]{M13}). 
Because of that, the robustness result only holds with an additional condition on the input separable matrix \cite[Th.~2]{M13}: the author requires that there are no duplicates nor near duplicates of the columns of $W$ in the data set which is a rather strong and not very reasonable assumption in practice; see the discussion in~\cite{G13}. 

Therefore, the main contributions of this paper are (i) the use of SDP for preconditioning near-separable matrices, and (ii) a general robustness analysis for preconditioned SPA. 
\end{remark}

\subsection{Notation} 

The set $\mathbb{S}^{r}_+$ denotes the set of  $r$-by-$r$ symmetric positive-semidefinite matrices. 
The eigenvalue of $A \in \mathbb{S}^{r}_+$ will be denoted  
\[
\lambda_{\max}(A) = \lambda_1(A) \geq \lambda_2(A) \geq \dots \geq \lambda_{r}(A) = \lambda_{\min}(A) \geq 0. 
\]
The singular values of a matrix $B \in \mathbb{R}^{r \times n}$ where $r \leq n$ are denoted
\[
\sigma_{\max}(B) = \sigma_1(B) \geq \sigma_2(B) \geq \dots \geq \sigma_r(B) = \sigma_{\min}(B) \geq 0. 
\]
For a matrix $X \in \mathbb{R}^{m \times n}$, we denote $X(:,j)$ or $x_j$ its $j$th column.

  


\section{Preconditioning} \label{precsec}

If a near-separable matrix $\tilde{M}$ satisfying Assumption~\ref{ass1} is pre-multiplied by a full rank matrix $Q$  such that $QW$ remains full column rank, then the new matrix $Q\tilde{M}$ also satisfies Assumption~\ref{ass1} where the matrix $W$ is replaced with $QW$, while the noise is replaced with $QN$. Therefore, if one can find a matrix $Q$ such that $QW$ is better conditioned than $W$ (and $QN$ is not too large), near-separable NMF algorithms applied on $Q\tilde{M}$ should be more robust against noise. 
We focus our analysis on the case $m = r$. 
We show in Section~\ref{ldr} how to use the preconditioning when $m > r$ using linear dimensionality reduction.  

A first straightforward and useful observation is that taking $Q = R W^{-1}$ where $R$ is any $r$-by-$r$ orthonormal matrix\footnote{
Actually, we only need $\kappa(QW) = 1$ so $Q$ can be a scaling of $R W^{-1}$.} (that is, $R^TR = I_r$) gives a matrix $QW$ which is perfectly conditioned, that is, $\kappa(QW) = 1$. Therefore, the preconditioning $Q$ should be as close as possible to a matrix of the form $R W^{-1}$. 
The main goal of this section is to show how to provably compute an approximation of $Q = R W^{-1}$ based on $\tilde{M}$ using semidefinite programming.

\subsection{Motivation: Preconditioning SPA} 

It is rather straightforward to analyze the influence of a preconditioning on SPA. 
\begin{corollary} \label{cor1} 
Let $\tilde{M}$ satisfy Assumption~\ref{ass1} with $m = r$ and let $Q \in \mathbb{R}^{r \times r}$.  
If $QW$ is full column rank, and 
\[
\epsilon \leq \mathcal{O} \left( \frac{\sigma_{\min}(QW)}{\sqrt{r} \sigma_{\max}(Q) \kappa^2(QW)} \right),
\] 
then SPA applied on matrix $Q\tilde{M}$ identifies indices corresponding to the columns of $W$ up to error  $\mathcal{O} \left(  \epsilon \, \kappa(Q) \,  \kappa(QW)^2  \right)$. 
\end{corollary}
\begin{proof}
Let us denote $\epsilon'$ the smallest value such that $||QN(:,j)||_2 \leq \epsilon'$ for all $j$,  $\sigma' = \sigma_{\min}(QW)$ and $\kappa' = \kappa(QW)$. 
By Theorem~\ref{th1}, if $\epsilon' \leq \mathcal{O} \left( \frac{\sigma'}{\sqrt{r} \kappa'^2} \right)$, SPA applied on matrix $Q\tilde{M}$ identifies the columns of matrix $QW$ up to error $\mathcal{O} \left( \epsilon' \kappa'^2 \right)$, that is, it identifies $r$ columns of $Q\tilde{M}$ such that for all $1 \leq k \leq r$ there exists $j$ in the extracted set of indices such that 
\[
\mathcal{O} \left( \epsilon' \kappa'^2 \right) \geq ||Q \tilde{M}(:,j) - Q W(:,k) ||_2  \geq  \sigma_{\min}(Q) ||\tilde{M}(:,j) - W(:,k) ||_2 . 
\]
This implies that the indices extracted by SPA allows us to identify the columns of $W$ up to error $\mathcal{O} \left( \frac{\epsilon' \kappa'^2}{ \sigma_{\min}(Q)} \right)$. 
Moreover, we have for all $j$ that 
\[
||QN(:,j)||_2 \leq \sigma_{\max}(Q) \epsilon, 
\] 
hence $\epsilon' \leq \sigma_{\max}(Q) \epsilon$ so that 
\[
\epsilon \leq \mathcal{O} \left( \frac{\sigma'}{\sqrt{r} \kappa'^2 \sigma_{\max}(Q)} \right) 
\Rightarrow \epsilon' \leq \mathcal{O} \left( \frac{\sigma'}{\sqrt{r} \kappa'^2} \right), 
\text{ while } 
\mathcal{O} \left( \frac{\epsilon' \kappa'^2 }{\sigma_{\min}(Q)} \right) \leq \mathcal{O} \left( \epsilon {\kappa'^2} \kappa(Q) \right).
\]  
\end{proof}

In particular, using $Q = W^{-1}$ (or any orthonormal transformation) gives the following result. 
\begin{corollary}  \label{cor2}
Let $\tilde{M}$ satisfy Assumption~\ref{ass1} with $m = r$. If $\epsilon \leq  \mathcal{O} \left( \frac{\sigma_{\min}(W)}{\sqrt{r}} \right)$, then SPA applied on the matrix $W^{-1}\tilde{M}$ identifies indices corresponding to the columns of $W$ up to error $\mathcal{O} \left( \epsilon \kappa(W) \right)$. 
\end{corollary} 
\begin{proof}
Taking $Q = W^{-1}$ in Corollary~\ref{cor1} gives $\sigma_{\min}(QW) = \kappa(QW) = 1$ while $\sigma_{\min}(Q) = \sigma_{\max}(W)^{-1}$, $\sigma_{\max}(Q) = \sigma_{\min}(W)^{-1}$, and $\kappa(Q) = \kappa(W)$. 
\end{proof}

Of course, the matrix $W$ is unknown, otherwise the near-separable NMF problem would be solved. 
In this section, we will show how to approximately compute $W^{-1}$ from $\tilde{M}$ using semidefinite programming.

\begin{remark}[Combining preconditioning with post-processing]
Note that combining the preconditioning with post-processed SPA (see Theorem~\ref{th2}) would not improve the error bound significantly (only up to a factor $\kappa(QW)$ which will be shown to be constant for our preconditioning); see Section~\ref{ne} for some numerical experiments. 
\end{remark}

\subsection{Minimum Volume Ellipsoid and SDP Formulation for Approximating $W^{-1}$}   

As explained in the previous section, we need to find a matrix $Q$ such that $\kappa(QW)$ is close to one. 
Let $A = Q^TQ \succ 0$. We have $W^T A W = (QW)^T (QW)$ so that $\sigma_i(W^T A W) = \sigma_i^2(QW)$ for all $i$, 
hence it is equivalent to find a matrix $A$ such that $\kappa(W^T A W)$ is close to one since it will imply that $\kappa(QW)$ is close to one, while we can compute a factorization of $A = Q^T Q$ (e.g., a Cholesky decomposition). Ideally, we would like that $A = (WW^T)^{-1} = W^{-T}W^{-1}$, that is, $Q = RW^{-1}$ for some orthonormal transformation $R$.

The central step of our algorithm is to compute the minimum volume ellipsoid centered at the origin containing all columns of $\tilde{M}$. 
An ellipsoid $\mathcal{E}$ centered at the origin in $\mathbb{R}^r$ is described via a positive definite matrix $A \in \mathbb{S}^r_{++}$ : 
\[
\mathcal{E} = \{ \ x \in \mathbb{R}^r \ | \ x^T A x \leq 1  \ \}. 
\]
The axes of the ellipsoid are given by the eigenvectors of matrix $A$, while their length is equal to the inverse of the square root of the corresponding eigenvalue. 
The volume of $\mathcal{E}$ is equal to $\det(A)^{-1/2}$ times the volume of the unit ball in dimension $r$. Therefore, given a matrix $\tilde{M} \in \mathbb{R}^{r \times n}$ of rank $r$, we can formulate the minimum volume ellipsoid centered at the origin and containing the columns $\tilde{m_i}$ $1 \leq i \leq n$ of matrix $\tilde{M}$ as follows
\begin{align}
\min_{A \in \mathbb{S}^r_+} \; \; 
\log \det(A)^{-1}  
\quad  \text{such that } \quad 
								 & \tilde{m_i}^T A \tilde{m_i} \leq 1 \quad \text{ for } i=1,2,\dots,n .    \label{SDPp} 
\end{align} 
This problem is SDP representable \cite[p.222]{BV04} (see also Remark~\ref{solveSDP}). 
Note that if $\tilde{M}$ is not full rank, that is, the convex hull of the columns of $\tilde{M}$ and the origin is not full dimensional, then the objective function value is unbounded below. Otherwise, the optimal solution of the problem exists and is unique \cite{john}. 
\begin{theorem} \label{precnoiseless}
For a separable matrix $\tilde{M} = WH + N$ 
satisfying Assumption~\ref{ass1} with $m = r$ and in the noiseless case (that is, $N = 0$ and $M = \tilde{M}$), the optimal solution of \eqref{SDPp} is given by $A^* = (WW^T)^{-1}$. 
\end{theorem} 
\begin{proof}
The matrix $A^* = (WW^T)^{-1}$ is a feasible solution of the primal~\eqref{SDPp}: In fact, for all $i$, 
\[
{m_i}^T A {m_i} = {h_i}^T W^T W^{-T}W^{-1} W{h_i} = ||h_i||_2^2 \leq ||h_i||_1^2 \leq 1. 
\]
The dual of \eqref{SDPp} is given by \cite[p.222]{BV04} 
\begin{equation} \label{dual}
\max_{y \in \mathbb{R}^n} 
\; \; 
\log \det \left(\sum_{i=1}^n y_i \tilde{m_i} \tilde{m_i}^T\right) - e^T y + r \quad  \text{ such that } \;  y \geq 0. 
\end{equation}
One can check that $y^* = [e_r; \, 0]$ is a feasible solution of the dual (with $\sum_{i=1}^n y_i \tilde{m}_i \tilde{m}_i^T = \sum_{i=1}^r w_i w_i^T = WW^T$ and $e^T y = r$) whose objective function value coincides with the one of the primal solution $A^* = (WW^T)^{-1}$ which is therefore optimal by duality. 
\end{proof}
Theorem~\ref{precnoiseless} shows that, in the noiseless case, the optimal solution of \eqref{SDPp} provides us with an optimal preconditioning for the separable NMF problem. \\ 



We now show that for sufficiently small noise, the optimal solution $A^*$ of \eqref{SDPp} still provides us with a good preconditioning for $\tilde{M}$, that is, $\kappa(W^T A^* W) \approx 1$. Intuitively, the noise $N$ perturbs `continuously' the feasible domain of \eqref{SDPp}, hence the optimal solution is only slightly modified for small enough perturbations $N$; see, e.g., \cite{R82}. 
In the following, we quantify this statement precisely. In other words, \emph{we analyze the sensitivity to noise of the optimal solution of the minimum volume ellipsoid problem}.

Let us perform a change of variable on the SDP \eqref{SDPp} using $A = W^{-T} C W^{-1}$ to obtain the following equivalent problem 
\begin{align}
C^* = \argmin_{C \in \mathbb{S}^r_+} \; \; & \log \det(C)^{-1} + \log \det(WW^T) \nonumber \\
\text{such that } \qquad 
								 & \tilde{m_i}^T \left( W^{-T} C W^{-1} \right) \tilde{m_i} \leq 1 \quad \text{ for } i=1, 2 , \dots, n .    \label{SDPc} 
\end{align} 
We have that $A^* = W^{-T} C^* W^{-1}$ where $A^*$ is the optimal solution of \eqref{SDPp}.  
Since our goal is to bound $\kappa(W^T A^* W)$ and $W^T A^* W = C^*$, it is equivalent to show that $C^*$ is well-conditioned, that is, $\kappa(C^*) \approx 1$. 


\subsection{Upper Bounding $\kappa(C^*)$} 

The case $r = 1$ is trivial since $C^*$ is a scalar hence $\kappa(C^*) = 1$ (the near-separable NMF problem is trivial when $r=1$ since all columns of $\tilde{M}$ are multiple of the column of $W \in \mathbb{R}^{m \times 1}$). 

Otherwise, since $\kappa(C^*) \geq 1$, we only need to provide an upper bound for $\kappa(C^*)$. 
Let us first provide a lower bound on $\det(C^*)$ which will be useful later. 
\begin{lemma}  \label{lemlwb} 
If $\tilde{M} = WH + N$ satisfies Assumption~\ref{ass1} with $m = r$, then the optimal solution $C^*$ of \eqref{SDPc} satisfies 
\begin{equation} \label{detClow}
\det(C^*) \geq \left( 1 + \frac{\epsilon}{\sigma_{\min}(W)} \right)^{-2r} .
\end{equation}
\end{lemma} 
\begin{proof} 
Let us consider the matrix  
\[
A = \alpha (WW^T)^{-1}, \text{ where } 0 < \alpha \leq 1. 
\]
Let us denote $\sigma = \sigma_{\min}(W)$. Note that $\sigma_{\max}((WW^T)^{-1}) = \sigma_{\max}^2(W^{-1}) = \frac{1}{\sigma^2}$. 
We have   
\begin{align*} 
 \tilde m_i^T A \tilde m_i = h_i^T W^T A W h_i + 2 n_i^T A W h_i + n_i^T A n_i 
& \leq \alpha \left(1 + 2 \epsilon ||W^{-T} h_i||_2 +  \frac{\epsilon^2}{\sigma^2}\right) \\
& \leq \alpha \left(1 + 2 \frac{\epsilon}{\sigma} +  \frac{\epsilon^2}{\sigma^2}\right) = \alpha \left(1 + \frac{\epsilon}{\sigma} \right)^{2}, 
\end{align*} 
since $\tilde{m}_i = Wh_i + n_i$, $||n_i||_2 \leq \epsilon$ and $||h_i||_2 \leq ||h_i||_1 \leq 1$ for all $i$. Therefore, taking
\[
\alpha = \left(1 + \frac{\epsilon}{\sigma} \right)^{-2}
\]
makes $A$ a feasible solution of \eqref{SDPp}. 
From the change of variable $A = W^{-T} C W^{-1}$, we have that 
\[
C = W^T A W = \left(1 + \frac{\epsilon}{\sigma} \right)^{-2} I_r 
\]
is a feasible solution of \eqref{SDPc} hence $\det(C^*) \geq \left( 1 + \frac{\epsilon}{\sigma} \right)^{-2r}$. 
\end{proof}

We can now provide a first upper bound on $\lambda_{\max}(C^*)$. 

\begin{lemma}  \label{lemmax11}
If $\tilde{M}$ satisfies Assumption~\ref{ass1} with $m = r$ then 
\[
\left( 1 - \frac{\sqrt{r} \epsilon}{\sigma_{\min}(W)} \right) \sqrt{\lambda} \leq {\sqrt{r}},  
\]  
where $\lambda = \lambda_{\max}(C)$ and $C$ is any feasible solution  of \eqref{SDPc}.  
\end{lemma} 
\begin{proof}
Let us denote $N' = W^{-1} N$ so that $||N'(:,i)||_2 \leq \sigma_{\max}(W^{-1})||N(:,i)||_2~\leq~\frac{\epsilon}{\sigma}$ where $\sigma = \sigma_{\min}(W)$.   
Any feasible solution $C$ of \eqref{SDPc} must satisfy for all $i$ 
\[
(Wh_i+n_i)^T W^{-T} C W^{-1} (Wh_i+n_i) 
= (h_i+n'_i)^T C  (h_i+n'_i)  \leq 1. 
\]
In particular, for all $1 \leq k \leq r$ , 
\begin{equation}
(e_k+n'_k)^T C (e_k+n'_k)  \leq 1, 
\end{equation}
 where $e_k$'s are the columns of the identity matrix (since $m_k = We_k + n_k = w_k + n_k$ for $1 \leq k \leq r$ by Assumption~\ref{ass1}). 
 Letting $C = B^TB$ so that $\sigma_{\max}(B) = \sqrt{\lambda}$, we have for all $1 \leq k \leq r$ 
 \begin{equation}  \label{normcolB}
||B(:,k)||_2 - \sigma_{\max}(B) \frac{\epsilon}{\sigma}  \leq ||B(e_k+n'_k)||_2 \leq 1 
 \Rightarrow 
||B(:,k)||_2 = ||Be_k||_2 \leq 1 + \sigma_{\max}(B) \frac{\epsilon}{\sigma}.  
 \end{equation} 
Therefore, 
\[ 
\sqrt{\lambda} = \sigma_{\max}(B) \leq \sqrt{r} \max_k ||B(:,k)||_2 \leq \sqrt{r} \left( 1 + \sigma_{\max}(B)   \frac{\epsilon}{\sigma} \right) = \sqrt{r} \left( 1 + \sqrt{\lambda} \frac{\epsilon}{\sigma} \right), 
\]
which gives the result. 
\end{proof}

We now derive an upper bound on $\lambda_{\max}(C^*)$ independent of $r$.  

\begin{lemma} \label{lammax}  
If $\tilde{M}$ satisfies Assumption~\ref{ass1} with $m = r \geq 2$ and 
\[
\epsilon \leq \frac{\sigma_{\min}(W)}{8 r \sqrt{r}},
\] 
then $\lambda = \lambda_{\max}(C^*) \leq 4$ where $C^*$ is the optimal solution of \eqref{SDPc}.
\end{lemma} 
\begin{proof}
Let us denote $\epsilon' = \frac{\epsilon}{\sigma_{\min}(W)}$ and let $C^* = B^TB$ for some $B$. By Lemma~\ref{lemmax11}, we have 
\[
\sqrt{\lambda} \leq \frac{ \sqrt{r} }{ 1 - \sqrt{r} \epsilon'  } \leq 2 \sqrt{r}
\]
since $\epsilon' \leq \frac{1}{8 r \sqrt{r}} \leq \frac{1}{2\sqrt{r}}$. 
Using Equation~\eqref{normcolB}, we obtain 
\begin{align*}
\sum_{k=1}^r \lambda_k(C^*) 
= \sum_{k=1}^r \sigma^2_k(B) 
= ||B||_F^2 
& = \sum_{k=1}^r ||B(:,k)||_2^2  \\
& \leq r \left( 1 + \sqrt{\lambda} \, \epsilon' \right)^2 \leq r \left( 1 +  2 \sqrt{r} \, \epsilon' \right)^2 \leq r \left( 1 +   \frac{1}{r} \right),  
\end{align*}
where the last inequality follows from $\epsilon' \leq \frac{1}{8 r \sqrt{r}}$: in fact,  
\[
\left( 1 +  2 \sqrt{r} \, \epsilon' \right)^2 \leq \left( 1 + \frac{1}{4r} \right)^2 = 1 + \frac{1}{2r} + \frac{1}{16r^2} \leq 1 + \frac{1}{r}. 
\] 
Since the maximum $x^*$ of the problem
\[
\max_{x \in \mathbb{R}^r_+} \prod_{k=1}^r x_k \quad \text{ such that } \sum_{k=1}^r x_k = \beta  \geq 0, 
\]
is unique and given by $x_k^* = \frac{\beta}{r}$ for all $k$, we have that the optimal solution $C^*$ of \eqref{SDPc} satifies 
\begin{equation} \label{bndsmax} 
\left( 1 +  \frac{1}{8r} \right)^{-2r} \leq \left( 1 + \epsilon' \right)^{-2r}
\leq \det(C^*) 
= 
\prod_{k=1}^r \lambda_k(C^*) 
\leq  
\lambda \left( \frac{r \left( 1 + \frac{1}{r} \right) - \lambda}{r-1} \right)^{r-1}, 
\end{equation} 
where the left-hand side inequality follows from Lemma~\ref{lemlwb} and $\epsilon' \leq \frac{1}{8r}$. Equivalently, 
\begin{equation} \nonumber 
(r-1)  
\left( 1 + \frac{1}{8r} \right)^{-\frac{2r}{r-1}}
\leq \lambda^{\frac{1}{r-1}} \left( r  + 1 - \lambda \right). 
\end{equation} 
In Appendix~\ref{app1}, we prove that this inequality implies $\lambda \leq 4$ for any integer $r \geq 2$. 
\end{proof}

We have now an upper bound on $\lambda_{\max}(C^*)$. In order to bound $\kappa(C^*)$, it is therefore sufficient to find a lower bound for $\lambda_{\min}(C^*)$.

\begin{lemma} \label{lemmin2} If $\tilde{M}$ satisfies Assumption~\ref{ass1} with $m = r \geq 2$ and 
\[
\epsilon \leq \frac{\sigma_{\min}(W)}{8 r \sqrt{r} }, 
\] 
then $\lambda_{\min}(C^*) \geq \frac{1}{10}$ where $C^*$ is the optimal solution of \eqref{SDPc}. 
\end{lemma}
\begin{proof}
Using the same trick as for obtaining Equation~\eqref{bndsmax} and denoting $\delta = \lambda_{\min}(C^*)$, we have 
\[ 
e^{-1/4} \leq \left( 1 +  \frac{1}{8r} \right)^{-2r}  \leq \det(C^*) 
 \leq \delta   \left( \frac{ r\left( 1+\frac{1}{r} \right) - \delta}{r-1} \right)^{r-1} 
 \leq \delta   \left( 1 + \frac{2}{r-1} \right)^{r-1} 
 \leq e^2 \delta , 
\] 
which implies $\delta \geq \frac{e^{-1/4}}{e^2} \geq \frac{1}{10}$. 
The first inequality follows from the nonincreasingness of $\left( 1 +  \frac{1}{8r} \right)^{-2r}$ and the limit to infinity being $e^{-1/4}$. 
The last inequality follows from the nondecreasingness of $\left( 1 + \frac{2}{r-1} \right)^{r-1}$ and 
the limit to infinity being $e^{2}$. 
\end{proof}

\begin{remark}
The lower bound for $\lambda_{\min}(C^*)$ can be improved to $\frac{1}{4}$ showing that 
\[
0.84 \leq \left( 1 + \frac{1}{8 r \sqrt{r}} \right)^{-2r} \leq  \det(C^*), 
\]
and
\[  
\det(C^*) \leq \delta \left( \frac{  r }{r-1}  \left( 1 + 2 \frac{\epsilon}{\sigma} \right)^2 \right)^{r-1} 
\leq \delta \left( \frac{  r }{r-1}  \left( 1 + \frac{1}{4 r \sqrt{r}} \right)^2 \right)^{r-1} 
 \leq 3 \delta
\]
 for any $r \geq 2$. 
\end{remark}

For sufficiently small noise level, we can now provide an upper bound for the condition number of~$C^*$. 
\begin{theorem} \label{lemkappa} 
If $\tilde{M}$ satisfies Assumption~\ref{ass1} with $m = r$ and 
\[
\epsilon \leq \frac{\sigma_{\min}(W)}{8 r \sqrt{r} }, 
\] 
then $\kappa(C^*) \leq 40$ where $C^*$ is the optimal solution of \eqref{SDPc}. 
\end{theorem} 
\begin{proof}
This follows from Lemmas~\ref{lammax} and~\ref{lemmin2}. 
\end{proof}

\subsection{Robustness of Preconditioned SPA}

The upper bound on $\kappa(C^*)$ (Theorem~\ref{lemkappa}) proves that the preconditioning generates a well-conditioned near-separable matrix, leading to more robust near-separable NMF algorithms. In particular, it makes preconditioned SPA significantly more robust to noise than SPA. 
\begin{theorem} \label{mainth} 
Let $M$ satisfy Assumption~\ref{ass1} with $m = r$ and let $Q \in \mathbb{R}^{r \times r}$ be such that $A^* = Q^TQ$ where $A^*$ is the optimal solution of \eqref{SDPp}.  
If $\epsilon \leq \mathcal{O} \left( \frac{\sigma_{\min}(W)}{r \sqrt{r} } \right)$, SPA applied on matrix $QM$ identifies indices corresponding to the columns of $W$ up to error  $\mathcal{O} \Big(  \epsilon \kappa(W)  \Big)$. 
\end{theorem}
\begin{proof}
This follows from Corollary~\ref{cor1}, Lemmas~\ref{lammax} and~\ref{lemmin2}, and Theorem~\ref{lemkappa}. In fact, let $Q$ be such that $A^* = Q^TQ$ where $A^*$ is the optimal solution of \eqref{SDPp}. Let also $C^* = W^T A^* W$ be the optimal solution of \eqref{SDPc}. 
We have that 
\[
(QW)^T QW = W^T Q^T Q W = W^T A^* W = C^*, 
\]
hence $\sigma_i(QW) = \sqrt{\lambda_i(C^*)}$ for all $i$ implying 
\[
\sigma' = \sigma_{\min}(QW) = \sqrt{\lambda_{\min}(C^*)} \geq \frac{1}{\sqrt{10}}, \qquad (Lemma~\ref{lemmin2})  
\] 
\[
\kappa' = \kappa(QW) = \sqrt{\kappa(C^*)} \leq \sqrt{40},  \qquad (Theorem~\ref{lemkappa}) 
\] 
\[
\sigma_{\min}(Q)  = \sigma_{\min}(QW W^{-1})  \geq \sigma_{\min}(QW) \sigma_{\min}(W^{-1}) \geq \frac{1}{\sqrt{10}\sigma_{\max}(W)}, 
\]
\[
\sigma_{\max}(Q) = \sigma_{\max}(QW W^{-1}) \leq \sigma_{\max}(QW) \sigma_{\max}(W^{-1}) \leq \frac{2}{\sigma_{\min}(W)}, \qquad (Lemma~\ref{lammax})  
\]
hence $\kappa(Q) \leq  2 \sqrt{10}  \kappa(W)$. 
\end{proof}

\begin{remark}[Tightness of the bound] \label{rem3} 
As explained in the introduction, the dependence on $\sigma_{\min}(W)$ for SPA is unavoidable. 
However, it is not clear whether the dependence on $r$ derived in Theorem~\ref{mainth} is optimal, this is a question for further research. 
(Note that tightness of the bounds in Theorem~\ref{th1} is also an open question.)
\end{remark}

\section{Linear Dimensionality Reduction} \label{ldr}

In practice, the input matrix $M \in \mathbb{R}^{m \times n}$ of rank $r$ usually has more than $r$ rows (that is, $m > r$) and the matrix $M$ is rank deficient. A simple and natural way to handle this situation is to use a linear dimensionality reduction technique as a pre-processing. This is a standard trick in signal processing and in particular in hyperspectral unmixing; see, e.g., \cite{Ma14} and the references therein. 

For example, one can use the truncated singular value decomposition (SVD). Given $\tilde{M} \in \mathbb{R}^{m \times n}$, an optimal rank-$r$ approximation of $\tilde{M}$ with respect to the Frobenius norm (and the 2-norm) is given by 
$\tilde{M}_r = U \Sigma V$ where $U \in \mathbb{R}^{m\times r}$ (resp.\@ $V \in \mathbb{R}^{r\times n}$) has orthonormal columns (resp.\@ rows) which are the first $r$ left (resp.\@ right) singular vectors of $\tilde{M}$, and $\Sigma = \diag\left(\sigma_1(\tilde{M}), \sigma_2(\tilde{M}), \dots, \sigma_r(\tilde{M})\right)$. 
Given the rank-$r$ truncated SVD $(U,\Sigma,V)$ of $M$, we can therefore use as the input near-separable matrix the matrix $\Sigma V \in \mathbb{R}^{r \times n}$. This linear dimensionality reduction amounts to replace the data points with their projections onto the linear space that minimizes the sum of the squared residuals.

\begin{remark}[Huge Scale Problems] \label{larger} If it is too expensive to compute the SVD, it is possible to use cheaper linear dimensionality reduction techniques such as random projections \cite{JW84, PRTV98}, or picking randomly a subset of the rows and then performing the SVD on the subset; see \cite{AB11} and the references therein. 
\end{remark}

\begin{remark}[Other Noise Models] Using the SVD for dimensionality reduction is ideal when the noise is Gaussian. When this is not the case (in particular if we have some information about the noise model), it would be beneficial to use other dimensionality reduction techniques. 
 \end{remark}

\subsection{Heuristic Preconditioning with SVD} \label{heurprec}

   It is interesting to observe that 
	\begin{enumerate}
	
	\item The pre-processing using the SVD described above is equivalent to pre-multiplying $\tilde{M}$ with $U^T$ because $U^T \tilde{M} = \Sigma V$. 
	
  \item 
Given the SVD $(U_W,\Sigma_W,V_W)$ of $W$, an optimal preconditioning is given by $Q = \Sigma_W^{-1} U_W^T$: in fact, 
$\kappa(Q W) = \kappa(\Sigma_W^{-1} U_W^T W)  = \kappa( V_W ) = 1$. 

\item The SVD of $M$ and $W$ are closely related (in fact, they share the same column space). Intuitively, the best  fitting linear subspace (w.r.t.\@ to the Frobenius norm) for the columns of $\tilde{M}$  should be close to the column space of $W$. 

\end{enumerate} 
Therefore, if one wants to avoid solving the SDP \eqref{SDPp}, a heuristic to approximately estimate $\Sigma_W^{-1} U_W^T$ is to use the SVD of $\tilde{M}$, that is, given the rank-$r$  truncated  SVD $(U,\Sigma,V)$ of $\tilde{M}$, use the preconditioning $\Sigma^{-1} U^T$ (note that this is prewhitening).  
We have observed that $U_W$ and $U$ are usually close to one another (up to orthogonal transformations), while $\Sigma_W^{-1}$ and $\Sigma$ are if the data points are well distributed in $\conv(W)$; see some numerical experiments in Section~\ref{ne}. 
However, in the next section, we show how to solve the SDP \eqref{SDPp} for large $n$ which makes this heuristic less attractive (unless $r$ is large).

\section{Active-Set Algorithm for Large-Scale Problems} \label{as} 


John \cite{john} proved that there exists a subset of the columns of $\tilde{M}$, say $\mathcal{S}$, with $r \leq |\mathcal{S}| \leq \frac{r(r+1)}{2}$ so that the problem \eqref{SDPp} is  equivalent to 
\begin{equation}
\min_{A \in \mathbb{S}^r_+} \quad \log \det(A)^{-1} \qquad \text{such that } \qquad \tilde{m}_i^T A \tilde{m}_i \leq 1 \; \text{ for } \; i \in \mathcal{S} .  \label{SDPS} 
\end{equation} 
For example, in the noiseless separable case, we have seen that there exists such a set with $|\mathcal{S}| = r$; see Theorem~\ref{precnoiseless}. (John \cite{john} also showed that if $\tilde{m}_i$ belongs to the convex hull of the other columns of $\tilde{M}$, then $i \notin \mathcal{S}$.)  
Although the set $\mathcal{S}$ is unknown in advance, we can try to make a good guess for $\mathcal{S}$, solve  \eqref{SDPS}, and then check whether all remaining columns of $\tilde{M}$ satisfy $\tilde{m}_i^T A \tilde{m}_i \leq 1$. If not, we can add some of them to the current set $\mathcal{S}$, remove columns with $\tilde{m}_i^T A \tilde{m}_i < 1$, and resolve \eqref{SDPS}; see Algorithm~\ref{sepnmf3} where $\eta$ is a parameter equal to the number of constraints kept in the active set.  
A similar active-set approach is described in \cite{SF04}. 
We propose to use SPA in order to find a good initial guess for $\mathcal{S}$; see Algorithm~\ref{initSPA}.  
Since SPA cannot extract more than $\min(m,n)$ indices of an $m$-by-$n$ matrix (because the residual has to be equal to zero after $\min(m,n)$ steps), we reinitialize SPA as many times as necessary, setting the columns corresponding to the indices already extracted to zero. 
\algsetup{indent=2em}
\begin{algorithm}[ht!]
\caption{Preconditioned SPA for Near-Separable NMF using SVD and an Active-Set Method \label{sepnmf3}}
\begin{algorithmic}[1]
\REQUIRE Near-separable matrix $\tilde{M} = WH + N \in \mathbb{R}^{m \times n}_+$ (Assumption~\ref{ass1}), number $r$ of columns to extract, number of active constraints $\frac{r(r+1)}{2}$$<$$\eta$$\leq$$n$, precision $\delta \geq 0$.  
\medskip 

\% \emph{(1) Linear Dimensionality Reduction using SVD} \\ 
\STATE Compute the truncated SVD 
$[U,\Sigma,V] \in \mathbb{R}^{m \times r} \times \mathbb{R}^{r \times r} \times \mathbb{R}^{n \times r}$ of $\tilde{M} \approx U \Sigma V^T$. 
\STATE Replace $\tilde{M} \leftarrow U^T \tilde{M} = \Sigma V^T$; \\ 
\% \emph{(2) Solve the SDP in the reduced space using an active-set method}
\STATE $\mathcal{S}$ = Algorithm 3$(\tilde{M}, \eta)$. \quad  \% \emph{Initialization using SPA} 
\STATE Compute the optimal solution $A$ of \eqref{SDPp} for matrix $\tilde{M}(:,\mathcal{S})$.  
\WHILE {$\max_i(\tilde{m}_i^T A \tilde{m}_i) \geq 1+\delta$} 
\STATE $\mathcal{S} \leftarrow \mathcal{S} \; \backslash \;  \{ i \in \mathcal{S} \ | \ \tilde{m}_i^T A \tilde{m}_i < 1 \}$.   \quad  \% \emph{Remove inactive constraints} 
\IF{$|\mathcal{S}| > \frac{r(r+1)}{2}$} 
  \STATE  \% \emph{Keep $\frac{r(r+1)}{2}$ `good' active constraints using SPA and $\max_i(\tilde{m}_i^T A \tilde{m}_i)$}  
	\STATE $\mathcal{S}$ = indices extracted by $\text{SPA}\left(\tilde{M}(:,\mathcal{S}),r\right)$; 
	\STATE Add $\frac{r(r+1)}{2}-r$ indices to $\mathcal{S}$  corresponding to the largest  $\tilde{m}_i^T A \tilde{m}_i$, $i \in \mathcal{S}$. 
	\ENDIF
\STATE Add $(\eta - |\mathcal{S}|)$ indices to the set $\mathcal{S}$   corresponding to the largest  $\tilde{m}_i^T A \tilde{m}_i$, $i \notin \mathcal{S}$. 
\STATE Compute the optimal solution $A$ of \eqref{SDPp} for matrix $\tilde{M}(:,\mathcal{S})$. 
\ENDWHILE \\
\% \emph{(3) Extract $r$ indices using SPA on the preconditioned matrix}
\STATE Compute $Q$ such that $A = Q^TQ$. 
\STATE $\mathcal{K}$ = SPA$(Q\tilde{M}, r)$ (see Algorithm~\ref{spa}).
\end{algorithmic}  
\end{algorithm} 

\begin{remark}[SDP with CVX] \label{solveSDP}
   We use CVX \cite{cvx, cvx2} to solve the SDP~\eqref{SDPp}: \vspace{-0.4cm}
\begin{lstlisting}
% Create and solve the minimum volumns ellipsoid centered at the origin
% and containing the columns of the r-by-n matrix Mtilde
cvx_begin quiet
    variable Q(r,r) symmetric 
    maximize( det_rootn( Q ) )
    subject to
        norms( Q * Mtilde , 2 ) <= 1;
cvx_end
\end{lstlisting} 
	The problem can be written directly in terms of the matrix $Q$ hence we do not need to perform a factorization at step 15 of Algorithm~\ref{sepnmf3}.    
 \end{remark}

\algsetup{indent=2em}
\begin{algorithm}[ht!]
\caption{Selection of an Initial Active Set using SPA \label{initSPA}}
\begin{algorithmic}[1]
\REQUIRE Matrix $M \in \mathbb{R}^{m \times n}$, number of indices to be extracted $K \leq n$. 
\ENSURE Set of indices $\mathcal{K} \subseteq \{1,2,\dots,n\}$ with $|\mathcal{K}| = K$. 
\medskip 

\STATE $R = M$; $\mathcal{K} = \{\}$. 
\WHILE{$|\mathcal{K}| < K$}
\STATE  $\mathcal{I} = \text{SPA}(R, K-|\mathcal{K}| )$. 
\STATE  $R(:,\mathcal{I}) = 0$. 
\STATE $\mathcal{K} = \mathcal{K} \cup \mathcal{I}$. 
\ENDWHILE
\end{algorithmic}  
\end{algorithm}

\section{Numerical Experiments} \label{ne}

In this section, we compare the following near-separable NMF algorithms on some synthetic data sets: 

\begin{enumerate}

\item \textbf{SPA}: the successive projection algorithm; see Algorithm~\ref{spa}.

\item \textbf{Post-SPA}: SPA post-processed with the procedure proposed by Arora et al.~\cite{Ar13}; see Section~\ref{spasec}.

\item \textbf{Prec-SPA}: preconditioned SPA; see Algorithm~\ref{sepnmf3} (we use $\delta = 10^{-6}$ and $\eta = \frac{r(r+1)}{2} + r$ for the remainder of the paper). 

\item \textbf{Post-Prec-SPA}: Prec-SPA post-processed  with the procedure proposed by Arora et al.~\cite{Ar13}; see Section~\ref{spasec}.

\item \textbf{Heur-SPA}: SPA preconditioned with the SVD-based heuristic; see Section~\ref{heurprec}: given the rank-$r$ truncated SVD $(U,\Sigma,V)$
of $\tilde{M}$, use the preconditioning $\Sigma^{-1} U^T$ and then apply SPA.

\item \textbf{VCA}: vertex component analysis, a popular endmember extraction algorithm proposed in~\cite{ND05}.

\item \textbf{XRAY}: recursive algorithm similar to SPA, but taking into account the nonnegativity constraints for the projection step~\cite{KSK12}. (We use in this paper the variant referred to as $max$.)

\end{enumerate} 
We also show that Prec-SPA can be applied to large-scale real-world hyperspectral images. \\ 

The Matlab code is available at \url{https://sites.google.com/site/nicolasgillis/}.  All tests are preformed using Matlab on a laptop Intel CORE i5-3210M CPU @2.5GHz 2.5GHz 6Go RAM. \\

\subsection{Synthetic Data Sets}  

In this section, we compare the different algorithms on some synthetic data sets.

\subsubsection{Middle Points Experiment} \label{mpsec}

We take $m = r = 20$, and $n = 210$. The matrix $W$ is generated using the \texttt{rand(.)}~function  of Matlab, that is, each entry is drawn uniformly at random in the interval $[0,1]$. 
The matrix $H =[I_r, H']$ is chosen such that the last columns of $M$ are located in between two vertices: $H'$ has exactly two non-zero entries in each row equal to 0.5. 
Hence, besides the 20 columns of $W$, there are $\binom{20}{2} = 190$ data points located in the middle of two different columns of $W$ for a total of 210 columns of $M$.  
The noise is chosen such that the columns of $W$ (that is, the first 20 columns of $M$) are not perturbed, while the 190 middle points are moved towards the outside of the convex hull of the columns of $W$: 
\[
N = \epsilon \; [0_{20 \times 20}, M(:,21)-{\bar{w}}, M(:,22)-{\bar{w}}, \dots, M(:,n)-\bar{w}], 
\]
where $M = WH$ and $\bar{w} = \frac{1}{r} \sum_i w_i$ is the vertex centroid of the convex hull of the columns of $W$. These are the same near-separable matrices as in \cite{GV12}. This makes this data set particularly challenging: for any $\epsilon  > 0$, no data point is contained in the convex hull of the columns of $W$. 

For each noise level (from 0 to 0.6 with step 0.01), we generate 100 such matrices, and Figure~\ref{xp1} reports the fraction of correctly extracted columns of $W$.  
\begin{figure}[ht!]
\begin{center}
\includegraphics[width=\textwidth]{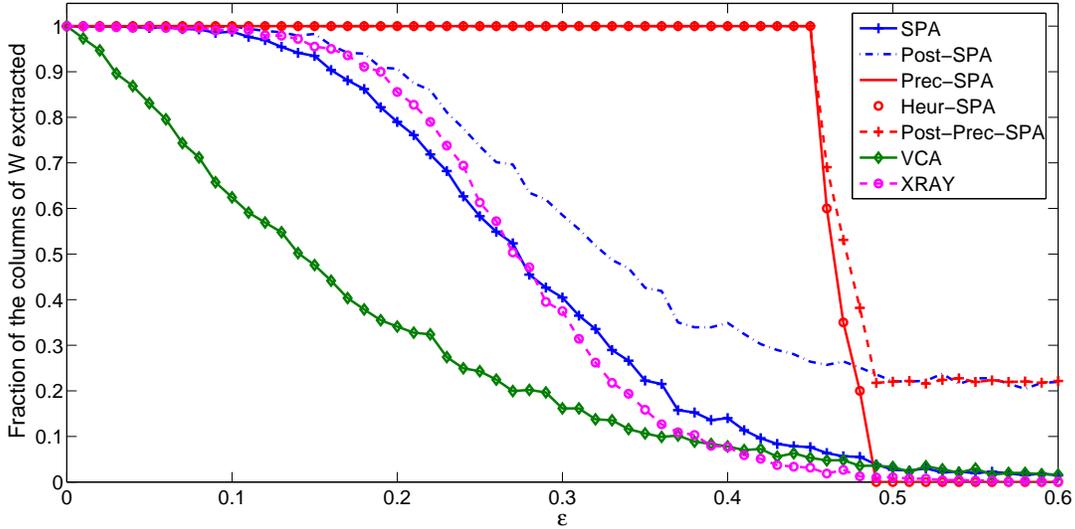}
\caption{Comparison of the different near-separable NMF algorithms on the `Middle Points' experiment.} 
\label{xp1}
\end{center}
\end{figure} 

We observe that 
\begin{itemize}

\item As observed in \cite{GV12}, VCA is not robust to noise. 

\item As observed in \cite{GL13},  XRAY and SPA perform similarly.  

\item Heur-SPA and Prec-SPA perform exactly the same. As explained in Section~\ref{ldr}, the reason is that the data points are rather well distributed and the SVD's of $W$ and $M$ are close to one another.

\item  Prec-SPA, Heur-SPA and Post-Prec-SPA are significantly more robust to noise than all other algorithms as they identify correctly all columns of $W$ for the largest noise levels; see also Table~\ref{robtim}. 
This confirms our theoretical findings that the preconditioning makes SPA significantly more robust to noise. 
This can be explained as follows: 
As long as the middle points remain inside the minimum volume ellipsoid containing the columns of $W$, the preconditioned algorithms perform perfectly (the solution of the SDP~\eqref{SDPp} actually remains unchanged). Moreover, when the middle points hit the border of the ellipsoid, there is a sharp change in the extracted columns (for $\epsilon \in [0.45,0.5]$). 

\item As predicted by Theorem~\ref{th2}, Post-SPA (resp.\@ Post-Prec-SPA) performs slightly better than SPA (resp.\@ Prec-SPA). 
Note that, for even larger noise levels which are not shown on Figure~\ref{xp1} (large enough so that the columns of $W$ are contained in the convex hull of the other columns), no column of $W$ will be extracted by Post-SPA and Post-Prec-SPA 
(more precisely, for $\epsilon \gtrsim 1.1$). 
\end{itemize}

Table~\ref{robtim} gives the robustness and the average running time of the different algorithms. Prec-SPA and Post-Prec-SPA are the slowest because they have to solve the SDP \eqref{SDPp}. XRAY is the second slowest because it has to solve nonnegative least squares problems at each step. 
\begin{table}[ht!] 
\begin{center}
\caption{Robustness (that is, largest value of $\epsilon$ for which the indicated percent of the columns of $W$ are correctly identified) and average running time in seconds of the different near-separable NMF algorithms on the `Middle Points' experiment.} 
\begin{tabular}{|c|c|c|c|}
\hline
  & Robustness (100\%) & Robustness (95\%) &  Time (s.) \\ \hline 
SPA  &  0.01 & 0.13 & $3.4^*10^{-3}$  \\ 
Post-SPA &  0.03 & 0.16 & $1.8^*10^{-2}$  \\  
Prec-SPA &  \textbf{0.45} & \textbf{0.45} & 2.8 \\ 
Heur-SPA &  \textbf{0.45} & \textbf{0.45}  & $2.4^*10^{-2}$ \\ 
Post-Prec-SPA &  \textbf{0.45} & \textbf{0.45}  &  2.8 \\ 
VCA & 0 &  0.02 & $2.3^*10^{-2}$  \\ 
XRAY & 0.01 & 0.15  & 0.69 \\ \hline
\end{tabular}
\label{robtim}
\end{center}
\end{table}

\subsubsection{Middle Points Experiment with Gaussian Noise} \label{mpsec2}

In the previous section, we explored a particularly challenging case where data points are located in between vertices and perturbed towards the outside of the convex hull.  
In this section, we explore a similar experiment with $m > r$ and using Gaussian noise. 
It is interesting to observe that running the `Middle Points' experiment  with $m > r$ gives the same outcome: in fact, by construction, the columns of the noise matrix $N$ belong to the column space $\col(W)$ of $W$. 
Therefore, the SVD pre-processing used by the different preconditionings does affect the geometry of the data points as it provides an exact decomposition of the data set since $\col(\tilde{M}) = \col(M) = \col(W)$ with dimension $r$. 

Therefore, in this section, in order to compare the performances of the different algorithms in case $m > r$, 
we also add Gaussian noise to the data points of the `Middle Points' experiment. 
More precisely, we use exactly the same settings as in Section~\ref{mpsec} except 
that $m = 30 > r = 20$ and the noise is generated as follows 
\[
N = 0.9 \cdot \epsilon \; [0_{30 \times 20}, M(:,21)-{\bar{w}}, M(:,22)-{\bar{w}}, \dots, M(:,n)-\bar{w}] 
+ 0.1 \cdot \epsilon \; Z , 
\]
where $\bar{w}$ is the vertex centroid of the columns of $W$ (see above), $Z(i,j) \sim$ N(0,1) for all $i,j$ with N(0,1) being the normal distribution with zero mean and standard deviation of one 
(we used the \texttt{randn(m,n)}~function of Matlab).

For each noise level (from 0 to 1 with step 0.01), we generate 100 such matrices, and Figure~\ref{xp2} reports the fraction of correctly extracted columns of $W$.  
\begin{figure}[ht!]
\begin{center}
\includegraphics[width=\textwidth]{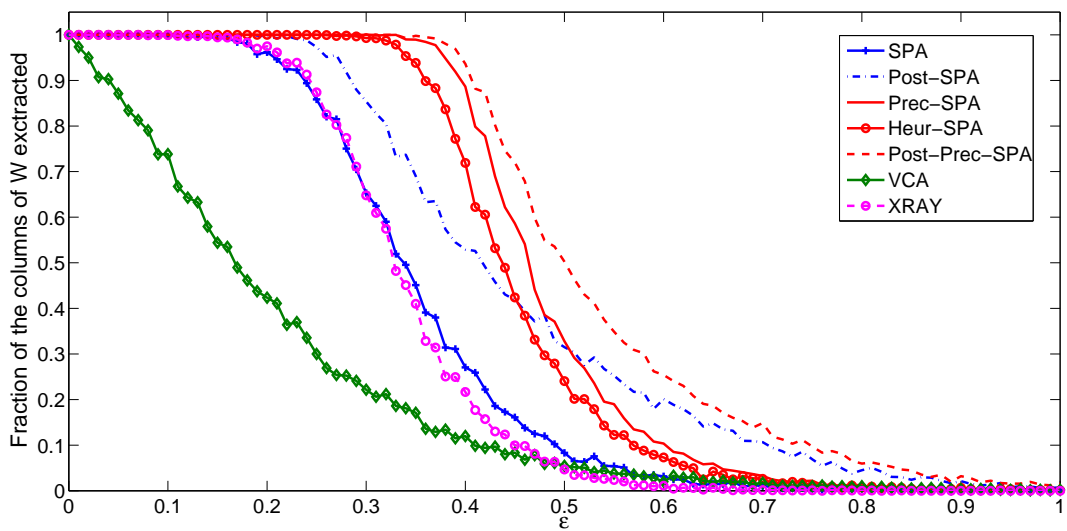}
\caption{Comparison of the different near-separable NMF algorithms on the `Middle Points' experiment with Gaussian noise.} 
\label{xp2}
\end{center}
\end{figure} 
Table~\ref{robtim2} gives the robustness and the average running time of each algorithm.  
\begin{table}[ht!] 
\begin{center}
\caption{Robustness (that is, largest value of $\epsilon$ for which the indicated percent of the columns of $W$ are correctly identified) and average running time in seconds of the different near-separable NMF algorithms on the `Middle Points' experiment with Gaussian noise.} 
\begin{tabular}{|c|c|c|c|}
\hline
     & Robustness (100\%) & Robustness (95\%) &  Time (s.) \\ \hline 
SPA           &  0.09   & 0.21   & $2.6^*10^{-3}$  \\ 
Post-SPA      &  0.18   & 0.27   & $2.2^*10^{-2}$  \\  
Prec-SPA      &  0.30   & 0.38   & 2.5 \\ 
Heur-SPA      &  0.25   & 0.34   & $3.1^*10^{-2}$ \\ 
Post-Prec-SPA &  \textbf{0.33}   & \textbf{0.40} &  2.5 \\ 
VCA           & 0       &  0.02  & $9.5^*10^{-2}$  \\ 
XRAY          & 0.04    &  0.21  & 0.5 \\ \hline
\end{tabular}
\label{robtim2}
\end{center}
\end{table}

It is interesting to observe that, in terms of robustness, there is now a clear hierarchy between the different algorithms, 
the one predicted by the theoretical developments: 
\[
\text{Post-Prec-SPA} 
\succ  
\text{Prec-SPA}
\succ  
\text{Heur-SPA}
\succ  
\text{Post-SPA}
\succ  
\text{SPA}
\succ  
\text{XRAY}
\succ  
\text{VCA}, 
\]
where $a \succ b$ indicates that $a$ is more robust than $b$. 
In particular, the post-processing is clearly advantageous, while the SDP-based preconditioning dominates the SVD-based heuristic variant.

\subsubsection{Hubble Telescope} \label{hubsec} 

In hyperspectral imaging, the columns of the input matrix are the spectral signatures of the pixels present in the image. Under the linear mixing model, the spectral signature of each pixel is equal to a linear combination of the spectral signatures of the materials present in the image (referred to as endmembers). The weights in the linear combinations represent the abundances of the endmembers in the pixels. 
If for each endmember, there is a least one pixel containing only that endmember, the separability assumption is satisfied: this is the so-called pure-pixel assumption un hyperspectral imaging. Therefore, under the linear mixing model and the pure-pixel assumptions, hyperspectral unmixing is equivalent to near-separable NMF and the aim is to identify one pure pixel per endmember; see \cite{GV12} and the references therein for more details. 

In this section, we analyze the simulated Hubble telescope hyperspectral image with 100 spectral bands and $128 \times 128$ pixels \cite{PPP06} (that is, $m = 100$ and $n = 16384$);  see Figure~\ref{hubblef}. 
It is composed of eight materials (that is, $r = 8$);  see Figure~\ref{materials}.  
\begin{figure}[ht!] 
\begin{center}
\begin{tabular}{cc}
\includegraphics[height=4.75cm]{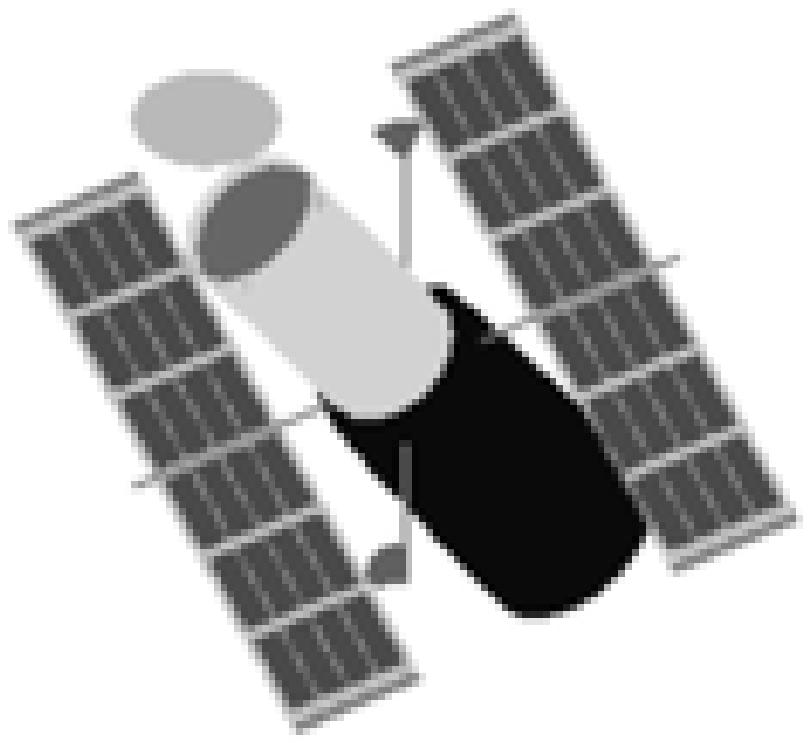} \hspace{1cm}
& \includegraphics[height=4.75cm]{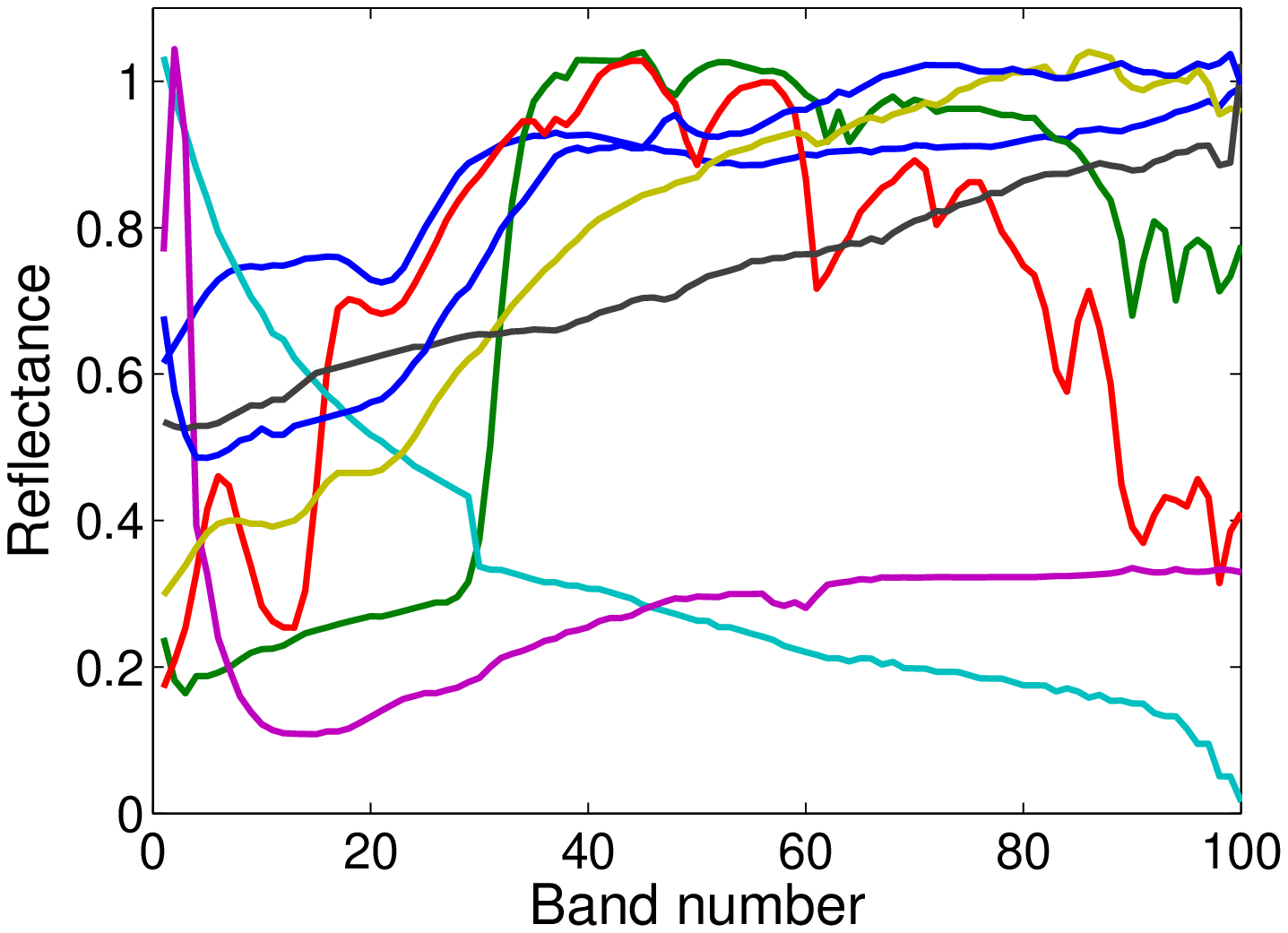} \\
\end{tabular}
\caption{On the left: sample image from the Hubble telescope hyperspectral image. On the right: spectral signatures of the eight endmembers.}
\label{hubblef}
\end{center}
\end{figure} 
\begin{figure}[H]
\begin{center}
\includegraphics[width=\textwidth]{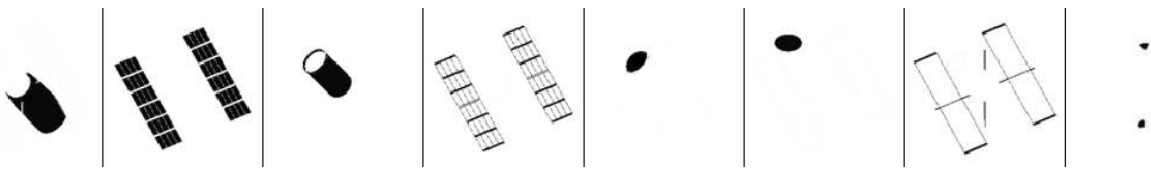}
\caption{The eight materials for the Hubble telescope data provided by the NASA Johnson Space Center. From left to right: aluminum, solar cell, green glue, copper stripping, honeycomb side, honeycomb top, black rubber edge and bolts.}
\label{materials}
\end{center}
\end{figure}

  On the clean image, all the SPA variants listed above are able to recover the eight materials perfectly\footnote{Note that VCA and XRAY are not able to identify all materials perfectly.}. 
	We add blur and noise as in \cite{PPP06} (point spread function on 5-by-5 pixels and with standard deviation of 1, and white Gaussian noise $\sigma = 1\%$ of the values of $M$ and Poisson noise $\sigma = 1\%$ of the mean value of $M$), and run all the algorithms. We then compute the mean-removed spectral angle (MRSA) between the true endmembers (of the clean image) and the extracted endmembers. Given two spectral signatures, $x, y \in \mathbb{R}^m$, the MRSA is defined as 
\begin{equation} \nonumber 
\phi(x,y) 
= \frac{100}{\pi}
\arccos \left( \frac{ (x-\bar{x})^T (y-\bar{y}) }{||x-\bar{x}||_2 ||y-\bar{y}||_2} \right) \quad \in \quad [0,100],  
\end{equation}
where, for a vector $z \in \mathbb{R}^m$,  $\bar{z} = \frac{1}{m} \left(\sum_{i=1}^m z(i)\right) e$ and $e$ is the vector of all ones. The MRSA measures how close two endmembers are (neglecting scaling and translation); 0 meaning that they match perfectly, 100 that they do not match at all.

  Table~\ref{mrsatim} reports the results, 	along with the running time of the different algorithms.  
		    \begin{table}[ht!]
\begin{center}
\caption{MRSA of the identified endmembers with the true endmembers, and running time in seconds of the different near-separable NMF algorithms.} 
\begin{tabular}{|c|cccccc|}
\hline
			&     SPA  & Post-SPA & Prec-SPA &  Heur-SPA & VCA & XRAY \\ \hline
Honeycomb side		& 	\textbf{6.51} & 6.94 & 6.94 & 6.94 &  48.70 &   45.78\\
 Copper Stripping&   26.83 & 7.46 & 7.46 & \textbf{7.44} &  48.98 &  46.93\\
Green glue & 2.09 & 2.09 & \textbf{2.03} & \textbf{2.03} &  48.30 &  42.70\\
Aluminum & \textbf{1.71} & 1.80 & 1.80 & 1.80  & 42.31 &  44.54\\
 Solar cell  &  \textbf{4.96} & 5.48 & 5.48 & 5.48  & 48.74 &   46.31\\
 Honeycomb top  &  2.34 & \textbf{2.07} & 2.30 & 2.30  & 49.90 &  53.23\\
 Black rubber edge  & 27.09 &  45.94 &  \textbf{13.16}  & \textbf{13.16}  & 47.21 &  46.55\\
 bolts &   \textbf{2.65} & \textbf{2.65} & \textbf{2.65} & \textbf{2.65} &  47.20 &  45.81\\ \hline 
	Average 	&	9.27   &  9.30  &   \textbf{5.23} &   \textbf{5.23} &   47.88 &   46.48 \\ \hline \hline 
 Time (s.)  &  0.05 & 1.45 & 4.74 & 2.18 & 0.37 & 1.53 \\ \hline  
\end{tabular}
\label{mrsatim}
\end{center}
\end{table} 
	Note that (i) we match the extracted endmembers with the true endmembers in order to minimize the average MRSA\footnote{We use the code available at \url{http://www.mathworks.com/matlabcentral/fileexchange/20328-munkres-assignment-algorithm}.}, and (ii) we do not include Post-Prec-SPA because it gave exactly the same solution as Prec-SPA. 
		
			\begin{figure}[ht]
\begin{center}
\includegraphics[width=\textwidth]{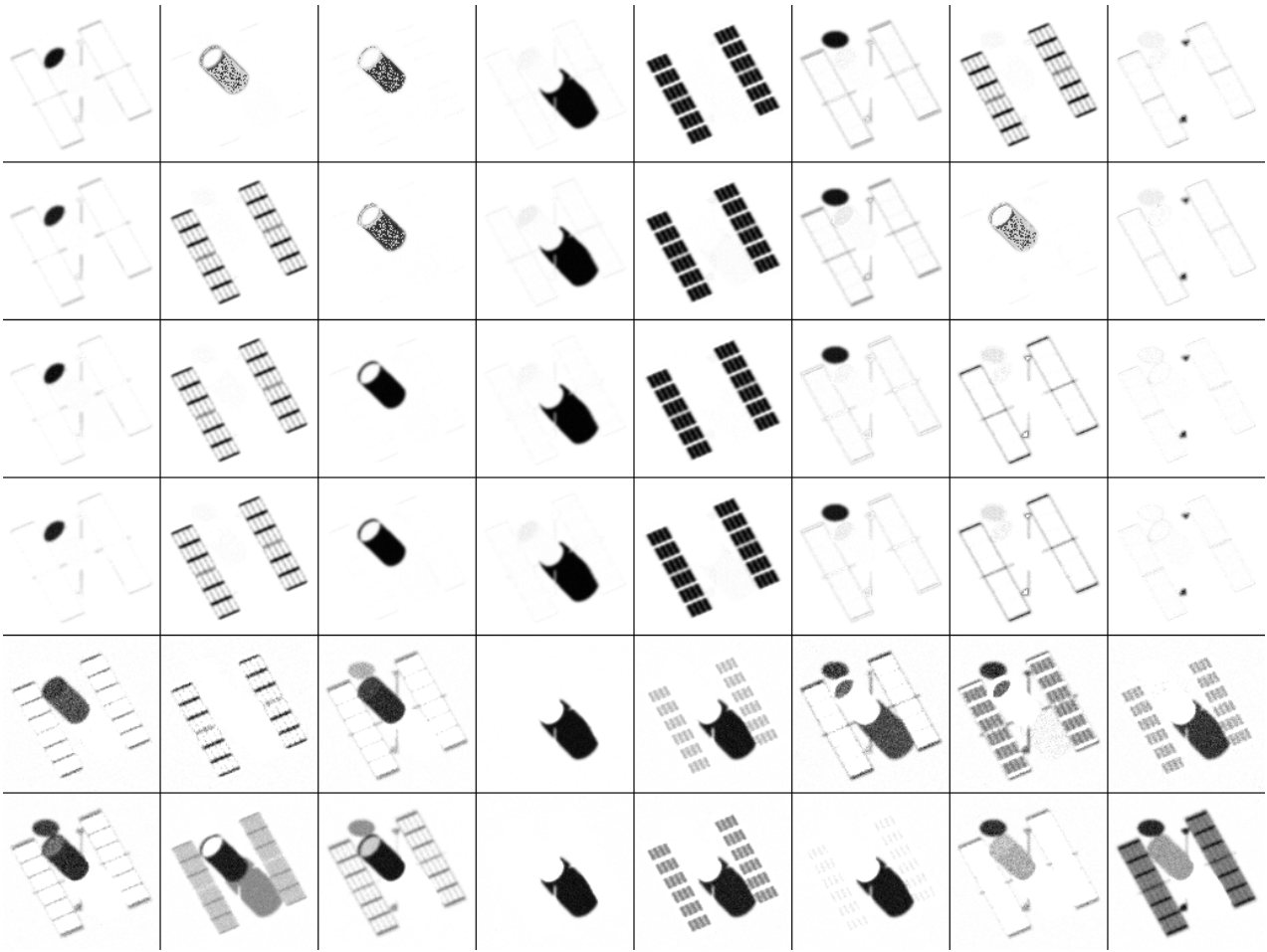}
\caption{The abundance maps corresponding to the extracted spectral signatures by the different algorithms. From top to bottom: SPA, Post-SPA, Prec-SPA, Heur-SPA, VCA and XRAY.} 
\label{abmap}
\end{center}
\end{figure} 
  We also computed the abundance maps corresponding to the extracted endmembers, that is, we solved the nonnegative least squares problem $H = \argmin_{X \geq 0} ||M-M(:,\mathcal{K})X||_F$ where $\mathcal{K}$ is the extracted index set by a given algorithm. This allows to visually assess the quality of the extracted endmembers; see  Figure~\ref{abmap}.

As for the synthetic data sets in Section~\ref{mpsec}, Prec-SPA and Heur-SPA perform the best. Moreover, for this data set, the running time of Prec-SPA is comparable to the one of Heur-SPA (because computing the SVD is more expensive; see  next section). Figure~\ref{pspa} displays the spectral signatures extracted by SPA and Prec-SPA, which shows that SPA is not able to identify one of the endmembers (in fact, one is extracted twice).  
\begin{figure}[ht!] 
\begin{center}
\begin{tabular}{cc}
\includegraphics[height=4.75cm]{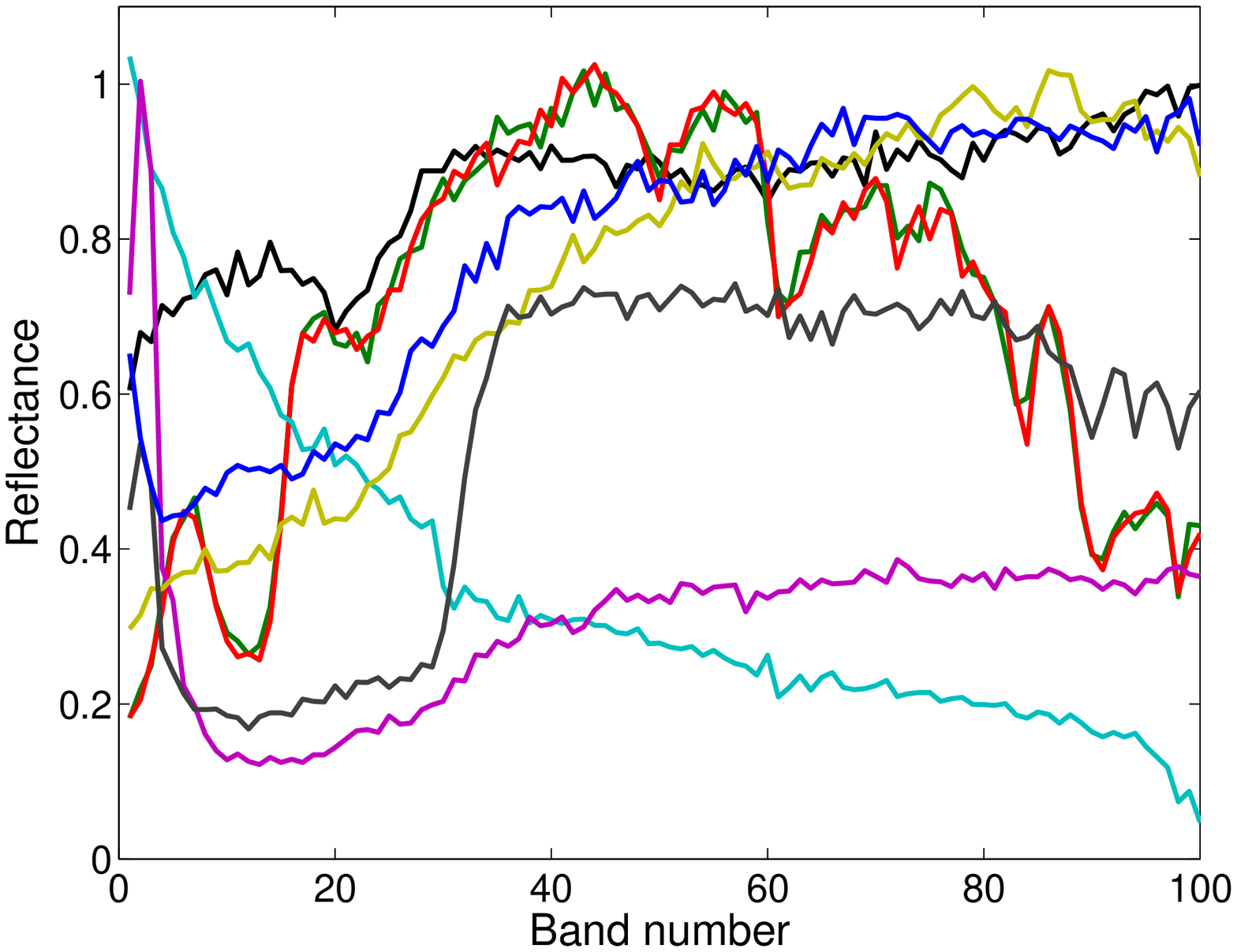} 
& \includegraphics[height=4.75cm]{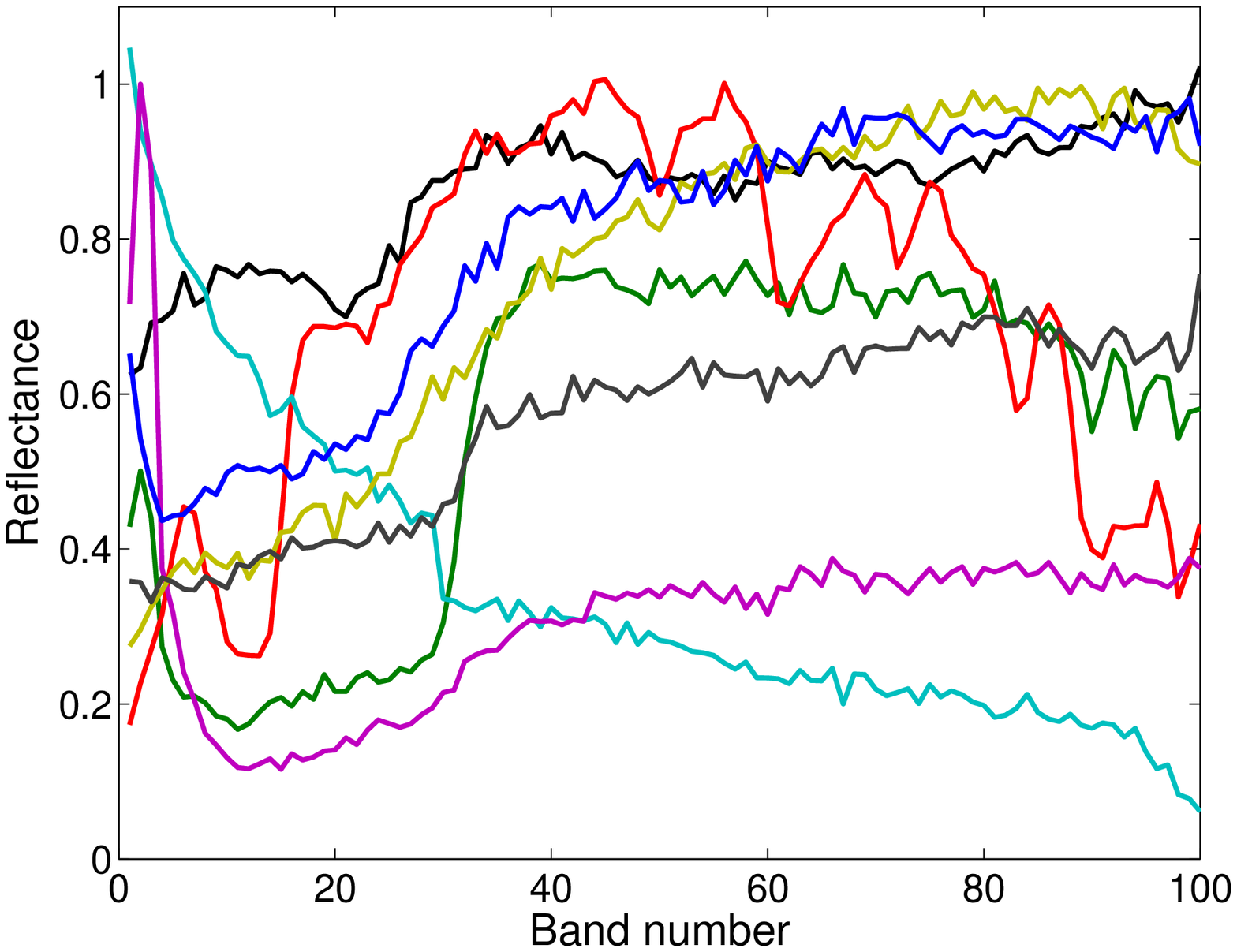} \\
\end{tabular}
\caption{Extracted endmembers by SPA (left) and Prec-SPA (right).}
\label{pspa}
\end{center}
\end{figure} 
It is also interesting to notice that, in this example, SPA performs slightly better than Post-SPA. (Although Post-SPA is guaranteed to identify a subset of data points whose convex hull has larger volume~\cite{Ar13}, it is not guaranteed, for large noise levels, that Post-SPA will approximate the columns of $W$ better.)

\subsection{Real-World Hyperspectral Images}

Assessing the quality of the extracted endmembers is rather challenging for real-world hyperspectral image because the true endmembers are unknown (in most cases, even the number of endmembers is unknown) while outliers and other artifacts in the image should be handled separately; see, e.g, 
\cite{Ma14} and the references therein.   
The aim of this section is therefore not to assess the performance of Prec-SPA on real-world hyperspectral images (this is out of the scope of this paper) but rather to show that the combination of the linear dimensionality reduction technique with the active-set method (that is, Algorithm~\ref{sepnmf3}) is applicable to this type of large-scale data sets. 
We report the running time of Algorithm~\ref{sepnmf3} for several popular hyperspectral images. We also report the difference between the solutions obtained with Prec-SPA and Heur-SPA (to illustrate the fact that they might behave quite differently in some cases); see Table~\ref{rw} for the numerical results. 
\begin{table}[ht!]
\begin{center} 
\caption{Running time in seconds for Prec-SPA (Algorithm~\ref{sepnmf3}) on several hyperspectral images: \emph{SVD} is the time in seconds needed to compute the truncated SVD of rank $r$, \emph{SDP} the time in seconds needed to solve the SDP~\eqref{SDPp} using the active-set method, \emph{Heur $\cap$ Prec} is the number of common indices extracted by Prec-SPA and Heur-SPA, and \emph{\# Active sets} is the number of changes in the active set needed for Algorithm~\ref{sepnmf3} to  terminate. }
\begin{tabular}{|c|c|c|c|c|c|c|c|}
\hline
			&      m  & n          & r      & SVD   & SDP  &  Prec  $\cap$ Heur  & \# Active sets \\ \hline
Urban$^*$  &    162 & 94249      & 6      & 7.3 & 1.7   & 6  & 3 \\
  &     			&  			       & 12     & 21.9 & 2.4   &  9  & 3 \\
  \hline 
San Diego$^{**}$ & 158   & 160000   & 8      & 17.1 & 1.5   & 8    & 2  \\ 
          &       &  			   & 16     & 32.1  & 7.2   &  13 &  4  \\ 
\hline 
 Cuprite$^{**}$ & 188 & 47750 & 15  & 11.0  & 6.6 &  6 & 5  \\ 
  &  &  							 & 30  & 18.2  & 58.3 & 3   &  4 \\ 
\hline
\end{tabular}

\begin{flushleft}
$^*$ \url{http://www.agc.army.mil}.

$^{**}$ \url{http://aviris.jpl.nasa.gov}. \vspace{-0.2cm}
\end{flushleft} 
\label{rw}
\end{center}
\end{table}

In all cases, less than five changes in the active set are necessary for Algorithm~\ref{sepnmf3} to terminate, 
which explains why the computational cost to solving the SDP is comparable to the one of computing the SVD. (Note that this will only be true for relatively small $r$.) It is interesting to observe that for some of these data sets Prec-SPA and Heur-SPA perform rather differently, especially when $r$ increases (for example, for the Cuprite data set with $r=30$ they only extract three common indices).

\section{Conclusion and Further Work}

In this paper, we have proposed a way to precondition near-separable NMF matrices using semidefinite programming. This in turn allowed us to robustify near-separable NMF algorithms. In particular, the preconditioning makes the popular successive projection algorithm (SPA)  provably more robust to noise, which we have illustrated on some synthetic data sets. Moreover, we showed how to apply the preconditioned SPA on real-world hyperspectral images using the SVD and an active-set method. 

For now, the preconditioning can be computed as long as (i) the SVD can be computed, and (ii) $r$ is not too large\footnote{With CVX, it took about 3 minutes to solve a problem with $m=r=50$ and $n=1000$ on a randomly generated matrix. For $r=100$, Matlab runs out of memory.}  (say $r \sim 50$). 
For larger data sets, several alternatives to the SVD are possible, such as random projections, and are interesting directions for further research; see also  Remark~\ref{larger}. For larger~$r$, it would be useful to develop faster SDP solvers, e.g., 
using first-order methods and/or 
using the particular structure of the minimum volume ellipsoid problem; see \cite{K93, K96, KY05, TY07, AST08} and the references therein. 

In the future, it would also be particularly interesting to assess the performance of preconditioned SPA on real-world problems (e.g., in hyperspectral unmixing and document classification), and evaluate the effect of the preconditioning on other algorithms (theoretically and/or practically). 

 \section*{Acknowledgments}
 
The authors would like to thank the editor and the reviewers for their insightful comments which helped improve the paper. \vspace{0.2cm}

{\small
\bibliographystyle{spmpsci} 
\bibliography{fastNMF}
}

\appendix 

\section{Proof for Lemma~\ref{lammax}} \label{app1} In this appendix, we prove the following: Let $r$ be any integer larger than $2$ and $\lambda \geq 0$ be such that  
\begin{equation} \label{shi}
(r-1)  
\left( 1 + \frac{1}{8r} \right)^{-\frac{2r}{r-1}}
\leq \lambda^{\frac{1}{r-1}} \left( r  + 1 - \lambda \right). 
\end{equation}
Then, $\lambda \leq 4$.

\begin{proof} 
The function $f(\lambda) = \lambda^{\frac{1}{r-1}} \left( r  +1 - \lambda \right)$ is nonincreasing for any $r \geq 2$ and for $\lambda \geq 2$; in fact, 
\[
\frac{d}{d \lambda}f(\lambda) 
= \frac{r+1-\lambda}{r-1} \lambda^{\frac{2-r}{r-1}} - \lambda^{\frac{1}{r-1}} 
= \lambda^{\frac{1}{r-1}} \left( \underbrace{\frac{r+1-\lambda}{r-1}}_{\leq 1} \underbrace{\lambda^{2-r}}_{\leq 1} - 1 \right)
 \leq 0 . 
\] 
Since the left-hand side of \eqref{shi} does not depend on $\lambda$, showing that the inequality above is violated for $\lambda = 4$ will imply that it is violated for any $\lambda \geq 4$ (hence $\lambda < 4$). It remains to show that 
\begin{equation} \nonumber 
(r-1)  
\left( 1 + \frac{1}{8r} \right)^{-\frac{2r}{r-1}}
> 4^{\frac{1}{r-1}} \left( r  - 3 \right) \quad  \text{ for any integer $r \geq 2$.} 
\end{equation} 
It can be checked numerically\footnote{The expression $\left( 1 + \frac{1}{8r} \right)^{-\frac{2r}{r-1}} - 4^{\frac{1}{r-1}} \left( r  - 3 \right)$ takes the values $4.78, 1.77, 1.18, 0.93, 0.80, 0.72, 0.66, 0.62, 0.59, 0.56, 0.55$ for $r = 2, 3, \dots, 12$ respectively.} that the inequality holds for any integer $r \leq 12$ hence it remains to show it holds for any $r \geq 13$. 
The inequality is equivalent to 
\[
\left( \frac{r-1}{r-3} \right)^{r-1} > 4  \left( 1 + \frac{1}{8r} \right)^{{2r}}. 
\]
We have $\left( 1 + \frac{1}{8r} \right)^2 = 1 + \frac{1}{4r} + \frac{1}{64r^2} \leq 1 + \frac{1}{3r}$ while $\frac{r-1}{r-3} < \frac{5}{4}$ for any $r \geq 13$ hence it is sufficient to show that 
\[
\left( \frac{r-1}{r-3} \right)^{r} > 5  \left( 1 + \frac{1}{3r} \right)^{r} \quad \text{ for $r \geq 13$}. 
\]
The right-hand side is an increasing function whose limit is given by $5 e^{1/3} \approx 6.978$. In fact, $e^x \geq \left( 1 + \frac{x}{n} \right)^n$ for all $x \geq 0$ and $n > 0$. 
The left-hand side is decreasing for $r\geq 4$. In fact, its derivative is given by 
\[
\frac{1}{r-3}
\left(\frac{r - 1}{r - 3}\right)^{r-1}  
\underbrace{\left( (r-1) \ln\left(\frac{r - 1}{r - 3}\right)
-  \frac{2 r}{(r - 3)}
\right)}_{h(r)} . 
\]
while using $\ln(x) \leq x - 1$ we have $\ln\left(\frac{r - 1}{r - 3}\right) \leq \frac{r - 1}{r - 3} - 1 = \frac{2}{r - 3}$ hence 
\[
h(r) \leq (r-1) \frac{2}{r - 3}
-  \frac{2 r}{r - 3}  = \frac{-2}{r - 3} \leq 0 \quad  \text{ for any } r \geq 4. 
\]
Finally, the limit of the left-hand side for $r \rightarrow +\infty$ is given by 
\begin{align*}
\lim_{r \rightarrow +\infty} \left( \frac{r-1}{r-3} \right)^{r} 
= \lim_{r \rightarrow +\infty} \left( 1 + \frac{2}{r-3} \right)^{r} 
& = \lim_{r \rightarrow +\infty} \left( \frac{r-1}{r-3} \right)^3 \lim_{r \rightarrow +\infty} \left( 1 + \frac{2}{r-3} \right)^{r-3} \\
& = e^2 \approx 7.389 > 5e^{1/3}. 
\end{align*}
\end{proof}

\end{document}